\pdfoutput=1
\documentclass[journal,twocolumn]{IEEEtran}
\usepackage{blindtext}
\usepackage{graphicx}
\usepackage{algorithm}
\usepackage{algorithmic}
\usepackage{amsthm}
\usepackage{amsmath}
\usepackage{caption}
\usepackage{subcaption}
\usepackage{dsfont}
\usepackage{pbox}
\usepackage[table]{xcolor}
\usepackage{color}
\usepackage[export]{adjustbox}
\usepackage{adjustbox}
\usepackage{subcaption}

\usepackage[shortlabels]{enumitem}
\usepackage{mathtools}

\newtheorem{theorem}{Theorem}
\newtheorem{lemma}[theorem]{Lemma}
\newtheorem{corollary}{Corollary}[theorem]
\definecolor{lightgray}{rgb}{0.83, 0.83, 0.83}	

\newenvironment{definition}[1][Definition]{\begin{trivlist}
		\item[\hskip \labelsep {\bfseries #1}]}{\end{trivlist}}

\begin{document}
%
\title{Algorithms and Improved bounds for online learning under finite hypothesis class}
%
%
%

\author{ Ankit Sharma\textsuperscript{1},\,
        Late C. A. Murthy\textsuperscript{2}

\thanks{1. Currently, Ankit Sharma  is working with a private organization as a Deputy Data Science Manager (R\&D). When this work was carried out, he was a final year student of M.Tech (Computer Science) at Indian Statistical Institute, Kolkata, West Bengal, India. \textbf{email:} 27ankitsharma@gmail.com}
\thanks{2. C.\! A.\! Murthy (1958-2018) was a senior Scientist and Professor at the Indian Statistical Institute (ISI), Kolkata, West Bengal, India. This work was carried out under his supervision when he was alive. \textbf{Web:} http://www.isical.ac.in/$\sim$murthy
\textbf{Wikipedia:} https://en.wikipedia.org/wiki/Chivukula\_Anjaneya\_Murthy}
}

\maketitle

\begin{abstract}

Online learning is the process of answering a sequence of questions based on the correct answers to the previous questions.
It is studied in many research areas such as game theory, information theory and machine learning.

There are two main components of online learning framework. First, the learning algorithm also known as the \textit{learner} and second, the \textit{hypothesis class} which is essentially a set of functions which learner uses to predict answers to the questions. 

Sometimes, this class contains some functions which have the capability to provide correct answers to the entire sequence of questions. This case is called \textit{realizable} case. And when hypothesis class does not contain such functions is called \textit{unrealizable} case.
 The goal of the learner, in both the cases, is to make as few mistakes as that could have been made by most powerful functions in hypothesis class over the entire sequence of questions. Performance of the learners is analysed by theoretical bounds on the number of mistakes made by them.

This paper proposes three algorithms to improve the \textit{mistakes bound} in the unrealizable case. Proposed algorithms perform highly better than the existing ones in the long run \textit{when most of the input sequences presented to the learner are likely to be realizable.}

\end{abstract}

\begin{IEEEkeywords}
Online learning, Weighted Majority with Consistent, Weighted Majority with Halving, Weighted Majority with SOA, Finite hypothesis class, Weighted Majority
\end{IEEEkeywords}

\IEEEpeerreviewmaketitle

\section{Introduction}

\subsection{Online Learning}

Online learning is the process of answering
a sequence of questions based on the correct answers to
the previous questions. The goal here is to make as few
mistakes as possible over an entire sequence of questions.
It is studied in many research areas such as game theory,
information theory and machine learning where settings of
online learning are similar to that of these areas.

There are two main components of online learning
framework. First, the learning algorithm also known as the
learner and second, the hypothesis class which is essentially
a set of functions. Learner tries to predict answers (labels)
to the questions using this set of functions.

This hypothesis class may be finite or infinite. Sometimes,
this class contains some functions which have the
capability to provide correct answers to the entire sequence
of questions. In this case, the goal of learner becomes
to identify these functions from the hypothesis class as
early as possible to avoid mistakes on further questions.
This setting, when function class contains some powerful
functions which can provide correct answers to the entire
sequence of questions, is called \textit{realizable case}.

Sometimes, it may not contain any such powerful
functions which can provide correct answers to the entire
sequence of questions. In such a case, learner has to rely
on the available functions in the hypothesis class and use
them intelligently to predict the answers. The goal of the
learner, therefore, becomes to make as few mistakes as that
could have been made by most powerful functions among
the available functions in hypothesis class. This setting, when hypothesis class
does not contain any such powerful functions which can
provide correct answers to the entire sequence of questions,
is called \textit{unrealizable} or \textit{agnostic case}.
\vspace{0.3cm}

Online learning is performed in the sequence of rounds where in each round $t$, the learner is given a question $x_t$ and the learner is required to predict the answer $p_t$ to this question. After the learner has  predicted the answer $p_t$, the right answer $y_t$ is revealed to the learner. Now, depending on the discrepancy between the predicted and right answer, learner suffers a loss $l(p_t,y_t)$. If learner suffers a loss i.e., it has made a wrong prediction, it is said to make a mistake. When the learner receives  right answer, it tries to improve the prediction mechanism for further questions. Thus, the goal of the learner is to make as few mistakes as possible over the entire sequence of questions.

\begin{algorithm}[h]
	\caption{Online Learning}
	\begin{algorithmic}
		\FOR{$i=1,2,\cdots$}
		\STATE receive question (feature vector) $x_t$ $\in$ $\chi$
		\STATE predict $p_t \in \mathcal{D}$
		\STATE receive the true answer $y_t$ $\in$ $\mathcal{Y}$
		\STATE suffer loss $l(p_t, y_t)$ 
		\ENDFOR
	\end{algorithmic}
	\label{alg:online_learning}
\end{algorithm}

Algorithm \ref{alg:online_learning} demonstrate a basic framework of online learning algorithm. In general, $\mathcal{D}$ and $\mathcal{Y}$ can be different. But when $\mathcal{D} = \mathcal{Y} = \{0,1\}$, we call it online classification. And in this case, naturally, we use 0-1 loss function: $l(x_t, y_t) = |p_t - y_t|$.
\vspace{0.5cm}

\subsubsection*{\textbf{Applications of Online Learning}}

The following are some real life scenarios where online learning finds its applications. \cite{Shalev-Shwartz:2012} 
\begin{enumerate}[(i)]

	\item \textbf{Online Ranking}: Here, the learner is required to rank the given list of elements. The learner is given a query $x_t$ $\in$ $\chi$, where $x_t$ is a list of $k$ elements (e.g. Documents). The learner is required to order these $k$ elements. Clearly, in this case $\mathcal{D}$ is the set of all permutations of these $k$ elements $\{1\cdots k\}$. So learner predicts one permutation $p_t$ out of the set D based on its knowledge deduced from the previous queries. Then, the learner is given the right answer  $y_t \in  \mathcal{Y} = \{1\cdots k\}$. This right answer corresponds to the document which best matches the query. This is an application of online learning in web search where online learning is used to order the documents retrieved by a search system with respect to the user input query. Right answer in this case becomes the document (web page) which user clicks on finally.
	
	\item \textbf{Prediction with expert advice}: 
	In this case, learner has a set of hypotheses $\mathcal{H}$ (e.g. experts or functions) and at each round, learner is required to predict the answer using one or some of these hypotheses from $\mathcal{H}$. Here, the challenge before learner is to use these experts intelligently so that it does not make more mistakes in the long run. To do it, learner uses the ``reward when correct and penalize when wrong" policy to weigh the experts. 
	
	\item \textbf{Choosing the best page replacement algorithm}:
	Operating system has many algorithms like FIFO, LRU, NRU etc to choose from to replace a page at any instance of time. Using a particular algorithm may not be  optimal all the time. Because, the performance of different algorithms may vary depending on the current state of the system.
	Therefore, prediction with expert advice form of online learning can be used to choose  best algorithm at $t^{th}$  instance of time. The available page replacement algorithms can be considered as a set of experts. These experts should be chosen intelligently by the learner to reduce the number of page faults in the long run.
	
	\item \textbf{Online email spam filtering}: This is another interesting application of online learning. In this case, learner is given an email feature vector $x_t$ and it is required to predict the label $\hat{y_t} \in \{0,1\}$ of email as spam or non-spam. Then, learner is given the correct label $y_t$ (marked spam or non-spam by the user) and thus learner updates its prediction mechanism for the next question.
	
\end{enumerate}

\subsection{Basic Settings and Terminologies}

This subsection discusses basic settings and terminologies of the online learning framework.\cite{Shalev-Shwartz:2012}
\begin{enumerate}[(i)]
	\item \textbf{Input sequence:} Input sequence contains $T$ points, where $T$ is finite. Another very important point is that the questions (input points) cannot be stored to be used in future. Once the algorithm has predicted the answer, point has to be discarded.
	
	\item \textbf{Binary classification:} The algorithms given in this work assumes only two classes as class $0$ and class $1$ i.e. $\mathcal{Y} = \{0,1\} $.
	
	\item \textbf{No statistical assumption on input sequence:} Classical statistical theory requires strong assumptions on statistical properties of data (e.g. sampled i.i.d. according to some unknown distribution.) But online learning does not require any statistical assumptions over the input sequence. A sequence can be deterministic, stochastic, or even adversarial adaptive to the learner's prediction mechanism.
	Since learner tries to deduce the knowledge from the previous correct answers, there must be some correlation between the past and the present rounds (points). If not, an adversary can make all the predictions of the learner wrong by just giving the opposite answer to what the learner had predicted.
	Therefore, we restrict the adversary to decide the answer to input question before the learner predicts.
	
	\item \textbf{Hypothesis class ($\mathcal{H}$):} We assume that the learner is armed with a class of hypotheses (functions). Out of these, some or all are used to predict the answer to the input point $x_t$.
	This class can be finite or infinite. But in this work, we assume that $\mathcal{H}$ is finite.
	
	\item \textbf{Realizable case:} The labels of the input sequence can always be assumed to be generated by a target hypothesis $h^*$ such that $h^* : \mathcal{X} \rightarrow \mathcal{Y}$. When this $h^* \in    \mathcal{H}$, we say that input sequence is realizable by the hypothesis class $\mathcal{H}$.
	
	\item \textbf{Unrealizable case (agnostic case)}: When we no longer assume that $h^* \in \mathcal{H}$, we say that input sequence is unrealizable by the hypothesis class $\mathcal{H}$.
	
	\item \textbf{Mistake bound $M_A(H)$:} In realizable case, mistake bound $M_A(H)$ is the maximal number of mistakes made by the algorithm $A$ on a sequence of examples which is generated by some $h^* \in \mathcal{H}$. Now, in this case objective is to design an algorithm which has minimal mistake bound $M_A(H)$.
	
	\item \textbf{Regret bound:} In unrealizable case, where we no longer assume that all the input points are labelled by some $h^* \in \mathcal{H}$, the number of mistakes made by the algorithm is compared with some best hypotheses $h \in \mathcal{H}$. This is termed as regret because this captures the regret of the algorithm, which measures how “sorry” the learner is, in retrospect, not to have followed the predictions of some hypothesis $h \in \mathcal{H}$.
	Formally, the regret of the algorithm relative to some $h \in \mathcal{H}$ when running on a sequence of $T$ points is defined as :
	$$\text{Regret}_T(h) = \sum\limits_{t = 1}^T l(p_t,y_t) -  \sum\limits_{t = 1}^T l(h(x_t),y_t)$$
	And the regret of the algorithm relative to the hypothesis class $\mathcal{H}$ is
	$$\text{Regret}_T(\mathcal{H})=\max_{h \in \mathcal{H}}\text{Regret}_T(h)$$
	In this case, objective becomes to design lowest possible regret algorithms. Low regret
	means $\text{Regret}_T(h)$ grows sub-linearly with $T$. i.e. $\text{Regret}_T(h) \rightarrow 0 $ as $ T \rightarrow \infty $.
	There are some other variations of these settings in online learning. For example, limited feedback \cite{Shalev-Shwartz:2012}, where after each round learner is given the loss value $l(p_t,y_t)$ but does not given
	the actual label $y_t$ of point $x_t$. Discussion about the algorithms in this setting is out of the scope of this work.
	\item \textbf{Online learnability of hypothesis class $\mathcal{H}$:} Let $\mathcal{H}$ be a hypothesis class and let $A$ be an online learning algorithm. Given any sequence $S = (x_1,h^*(y_1)), \cdots \cdots (x_T ,h^* (y_T )) $, where $T$ is a finite integer and $h^* \in \mathcal{H}$, let $M_A(S)$ be the number of mistakes $A$ makes on the sequence $S$. We denote by $M_A(\mathcal{H})$ the supremum of $M_A(S)$ over all sequences of the above form. A bound of the form $M_A(\mathcal{H}) \leq B \leq \infty $ is called a mistake bound. We say that a hypothesis class $\mathcal{H}$ is online learnable if there exists an algorithm $A$ for which $M_A(\mathcal{H}) \leq B < \infty $.
	
	\item \textbf{Best hypotheses:} $\mathcal{H}$ is the set of some hypotheses. Given an input sequence, if we use all of
	them one by one, some hypotheses may make more mistakes than the others. Then, those hypotheses which make least number of mistakes are called best hypotheses. \textit{There can be more than one such hypotheses.}
	
	\item \textbf{Ldim($\mathcal{H}$):} This is a dimension of hypothesis classes that characterizes the best possible
	achievable mistake bound for a particular hypothesis class. This measure was proposed by Nick Littlestone \cite{Littlestone1988} and referred to as Ldim($\mathcal{H}$). Before explaining Ldim($\mathcal{H}$), one definition needs to be given.

	\begin{definition}[Definition 1.1]
		\textbf{$\mathcal{H}$ Shattered tree \cite{Shalev-Shwartz:2012}:}
		A shattered tree of depth $d$ is a sequence of instances $(v_1, v_2,\cdots, v_{2^d-1})$ in $\mathcal{X}$ such that for all labelling $(y_1,y_2,\cdots$ $, y_d) \in \{0,1\}^d$, $ \exists\, h^* \in \mathcal{H}$ such that $ \forall \,t \in \left[ d \right]$ we have $h(v_{i_t}) = y_t$, where  $i_t = 2^{t-1} + \sum\limits_{j = 1}^{t-1} y_j 2^{t-1-j}$.
	\end{definition}

	\begin{definition}[Definition 1.2] 
		\textbf{Littlestone's dimension (Ldim($\mathcal{H}$)) \cite{Shalev-Shwartz:2012}:}
		Ldim($\mathcal{H}$) is the maximal integer $T$ such that there exist a shattered tree of depth $T$.
	\end{definition}

	\begin{table}[h!]
		\centering
	
		\caption{Predictions of $\mathcal{H} = \{h_1,h_2,h_3,h_4\}$ on the sequence of examples $v_1,v_2,v_3$.}
		\label{tab:ldim}
		\scalebox{1.3}{
		\begin{tabular}{ |c|c|c|c|c|}
			\hline
			\rowcolor{lightgray}
			
			& $h_1$ & $h_2$ & $h_3$ & $h_4	$ \\
			\hline
			$v_1$ & 0    & 0    & 1    & 1    \\
			$v_2$ & 0    & 1    & *    & *    \\
			$v_3$ & *    & *    & 0    & 1    \\
			\hline
		\end{tabular}
		}
	\end{table}

	\begin{figure}[h!]
		\centering
		\includegraphics[scale=1]{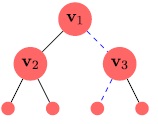}
		\caption{\small Shattered tree of depth 2}
		\caption*{The dashed blue path corresponds to the sequence of examples ($(v_1,1),(v_3,0)$). This tree is shattered by the hypothesis class $\mathcal{H}$ = \{$h_1,h_2,h_3,h_4$\} where the predictions of each hypothesis in $\mathcal{H}$ on the instances $v_1,v_2,v_3$ is given in the table \ref{tab:ldim}. Here, * means it can be $0$ or $1$.}
	\end{figure}
	
	Ldim($\mathcal{H}$) is very crucial combinatorial measure in online learning as VC dim($\mathcal{H}$) \cite{vapnik1994measuring} is in PAC learning \cite{kearns1994toward} because it provides the lower bound on the number of mistakes in the realizable case as the following lemma states:
	
	\begin{definition}[Lemma 1.1]:
		No algorithm can have a mistake bound strictly smaller than Ldim($\mathcal{H}$), namely, $\forall A$, $M_A(\mathcal{H}) \geq$ Ldim($\mathcal{H}$) \cite{Shalev-Shwartz:2012}.
	\end{definition}
	\begin{proof}
		Let $T$ = Ldim($\mathcal{H}$) and let $v_1,\cdots ,v_{2^T-1}$ be a sequence that satisfies the requirements in the definition of Ldim. If the environment
		sets $x_t = v_{i_t}$ and $y_t = 1 - p_t$ $ \forall t \in [T]$, then the learner makes $T$
		mistakes while the definition of Ldim implies that there exists a hypothesis
		$ h\in \mathcal{H} $ such that $y_t = h(x_t)$ for all $t$.\\
	\end{proof}
	
	We have the following relation among  Ldim($\mathcal{H}$), VC-dim($\mathcal{H}$) and $\text{log}_2(|\mathcal{H}|$).
	
	\begin{definition}[Lemma 1.2]:
		VC-dim($\mathcal{H}) \leq$ Ldim($\mathcal{H}) \leq \text{log}_2(|\mathcal{H}|)$ \cite{Shalev-Shwartz:2012}.
	\end{definition}
	
	\begin{enumerate}[(a)]
		\item VC-dim($\mathcal{H}) \leq$ Ldim($\mathcal{H})$ and this gap can be arbitrarily large. 
		
		\begin{proof}
			Suppose VC-dim($\mathcal{H}) = d $ and let $x_1,\cdots,x_d$ be a shattered set. We now construct a complete binary tree of instances $v_1,\cdots,v_{2^d-1}$, where all nodes at depth $i$ are set to be $x_i$. Now, the definition of shattered sample clearly implies that we got a valid shattered tree of depth d, and we conclude that VC-dim($\mathcal{H}) \leq$ Ldim($\mathcal{H})$.
			
			\begin{figure}[h!]
				\centering
				\includegraphics[scale=1]{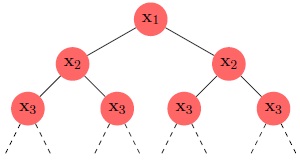}
				\caption{\small Constructing a shattered tree from a shattered sequence ($x_1,\cdots,x_d$)}
			\end{figure}
			
			Now the following example shows that the gap can be arbitrarily large.\\
			Example 1.1 Let $\mathcal{X} = [0,1]$ and $\mathcal{H} = \{x \rightarrow \mathbf{1}_{[x \geq a]}: a \in [0,1]\}$, namely, $\mathcal{H}$ is the class of threshold on the segment $[0,1]$. Then, Ldim($\mathcal{H}) = \infty$.\\
			Here the gap between two quantities is infinity as $\mathcal{H}$ has Ldim =  $\infty$ and VC-dim = $1$.
		\end{proof}
		
		\item Ldim$(\mathcal{H}) \leq \text{log}_2(|\mathcal{H}|)$
		
		\begin{proof}
			Any tree that is shattered by $\mathcal{H}$ has depth at most $\text{log}_2(|\mathcal{H}|)$. Therefore  $$\text{Ldim}(\mathcal{H}) \leq \text{log}_2(|\mathcal{H}|)$$.
		\end{proof}
	\end{enumerate}
\end{enumerate}

\section{Existing methodologies}
In this section, we will be describing all the related methods to this work in the
literature of online learning.
Before we begin, let us state the goal of a learning algorithm in online learning framework. 

In realizable case, a learner should have minimal mistake bound $M_A(H)$ and in agnostic case, it should have $\text{Regret}_T(\mathcal{H})$ growing sub-linearly with $T$. Here, sub-linear implies that the difference between the
$average$ loss of the learner and the average loss of the best hypothesis in $\mathcal{H}$ tends to zero as $T$ goes to infinity.

Now, based on the size of hypothesis class and realizability/unrealizability of the input sequence by hypothesis class $\mathcal{H}$, all the available methods in online learning, in standard settings, can be divided into the following categories:
\begin{enumerate}[A.]
	\item Finite hypothesis class and realizable case
	\item Finite hypothesis class and unrealizable case
	\item Infinite hypothesis class and both realizable and unrealizable cases
\end{enumerate}

We describe all the methods for each category one by one as follows (for A and B), except the last one (C) as it is out of scope of this work.

\subsection{\textbf{Finite hypothesis class and realizable case}}
\label{subsec:FHRC}

In this setting, we assume that all target labels are generated by some target hypotheses $h^* \in \mathcal{H}$ such that $y_t = h^*(x_t),  \forall t$. We also assume that $|\mathcal{H}|$ is finite.

The following are three algorithms in this  setting :

\subsubsection{\textbf{Consistent}}
This is given in Algorithm \ref{alg:Consistent}. \cite{Shalev-Shwartz:2012}
\begin{algorithm}[h]
	\caption{Consistent}
	\begin{algorithmic}
		\STATE \textbf{Input:} A finite hypothesis class $\mathcal{H}$
		\STATE \textbf{Initialize:} $V_1 = \mathcal{H} $
		\FOR{$t = 1,2,\cdots$}
		\STATE receive $x_t$
		\STATE choose any $h \in V_t$
		\STATE predict $p_t = h(x_t)$
		\STATE receive true answer $y_t = h^*(x_t)$
		\STATE update $V_{t+1} = \{ h \in V_t : h(x_t) = y_t\}$ 
		\ENDFOR
	\end{algorithmic}
	\label{alg:Consistent}
\end{algorithm}

$ Consistent $ algorithm is a basic algorithm which uses very naive approach to find best hypothesis. It chooses any hypothesis randomly from the available hypothesis set to predict the label of the point. But for the future rounds, it retains only those hypotheses which have predicted correctly till the current point. By this way, if the algorithm makes a mistake in any round, it discards at least one hypothesis from the hypothesis class $\mathcal{H}$. i.e. after making a mistake in $t^\text{th}$ round $$|H^{t+1}| \leq |H^t| - 1.$$

\subsubsection*{Analysis of Consistent}
The $ Consistent $ algorithm maintains a set, $V_t$, of all the hypotheses which are consistent with $(x_1,y_1),\cdots,(x_{t−1},y_{t−1})$. This set is often called the version space. It, then, picks any hypothesis from $V_t$ and predicts according to this hypothesis.
It is clear that whenever $ Consistent $ makes a mistake; at least one hypothesis is removed from the $V_t$. So after making $M$ mistakes, $|V_t| = |\mathcal{H}| - M $. 

Note that $\mathcal{H}$ is never empty because of the realizability assumption that $ h^* \in \mathcal{H}$. So, in the worst case, $ Consistent $ can make at most $|\mathcal{H}| - 1$
mistakes and it will not make any mistake in the best case. The best case corresponds to the situation when the learner gets hold of the best hypothesis in the very beginning itself.
Therefore, based on this discussion, we have the following corollary stating the mistake bound of the $ Consistent $.

\begin{corollary}:
	Let $\mathcal{H}$ be a finite hypothesis class. The $ Consistent $ algorithm enjoys the following mistake bound \cite{Shalev-Shwartz:2012}
	$$M_{Consistent}(\mathcal{H}) \leq |\mathcal{H}| - 1 $$
\end{corollary}

\subsubsection{\textbf{Halving}}
This is given in Algorithm \ref{alg:halving}. \cite{Shalev-Shwartz:2012}
\begin{algorithm}[H]
	\caption{Halving}
	\begin{algorithmic}
		\STATE \textbf{Input:} A finite hypothesis class $\mathcal{H}$
		\STATE \textbf{Initialize:} $V_1 = \mathcal{H} $
		\FOR{$t = 1,2,\cdots$}
		\STATE receive $x_t$
		\STATE predict $p_t = \operatorname{arg\,\max}_{r \in \{0,1\}} |\{ h \in V_t : h(x_t) = r\}|$
		\\(in case of a tie predict $p_t = 1$)
		\STATE receive true answer $y_t = h^*(x_t)$
		\STATE update $V_{t+1} = \{ h \in V_t : h(x_t) = y_t\}$ 
		\ENDFOR
	\end{algorithmic}
	\label{alg:halving}
\end{algorithm}

\subsubsection*{Analysis of Halving}

In $ Consistent $ algorithm, we just use arbitrary single hypothesis to predict the label. A better idea would be to take the majority vote and then decide. This will improve the chances of correct prediction and it will also enable us to discard at least half of the hypotheses if algorithm makes a mistake in any round.

At each round, it partitions the hypothesis class into two sets. One set consists of all those hypotheses which predict $0$ and other contains those which predict $1$ for the given point $x_t$ in the $t^{th}$ round. Then, $ Halving $ chooses prediction of the set which has larger cardinality. When correct answer is revealed to the learner, it discards the set whose prediction are not same as the correct answer. Thus at any round if algorithm makes a mistake, it can safely discard at least half of the hypotheses.

We have the following theorem stating the mistake bound of the $ Halving $ algorithm :
\begin{theorem}
	Let $\mathcal{H}$ be a finite hypothesis class. The Halving algorithm enjoys the mistake bound $M_{halving}(H) \leq \text{log}_2(|\mathcal{H}|)$.\cite{Shalev-Shwartz:2012}
\end{theorem}

\begin{proof}
	We simply note that whenever \textit{Halving} makes a mistake we have $|V_{t+1}| \leq |V_t|/2$. Therefore, if $M$ is the total number of mistakes upto round $t$ then we have,$$1 \leq |V_{t+1}| \leq |\mathcal{H}|/2^M$$
	Now rearranging the terms, we have
	$$1 \leq |\mathcal{H}|/2^M$$
	$$2^M \leq |\mathcal{H}|$$
	$$M \leq \text{log}_2(|\mathcal{H}|)$$
\end{proof}

\subsubsection{\textbf{Standard Optimal Algorithm (SOA)}}

This is given in Algorithm \ref{alg:soa} .\cite{Shalev-Shwartz:2012} :
\begin{algorithm}[h]
	\caption{SOA}
	\begin{algorithmic}
		\STATE \textbf{Input:} A finite hypothesis class $\mathcal{H}$
		\STATE \textbf{Initialize:} $V_1 = \mathcal{H} $
		\FOR{$t = 1,2,\cdots$}
		\STATE receive $x_t$
		\STATE for $r \in \{0,1\}$ let $ V_t^{(r)} = \{ h \in V_t : h(x_t) = r\}$ 	
		\STATE predict $p_t = \operatorname{arg\,\max}_{r \in \{0,1\}} \text{Ldim}(V_t^{(r)})$
		\\(in case of a tie predict $p_t = 1$)
		\STATE receive true answer $y_t = h^*(x_t)$
		\STATE update $V_{t+1} = \{ h \in V_t : h(x_t) = y_t\}$ 
		\ENDFOR
	\end{algorithmic}
	\label{alg:soa}
\end{algorithm}

\subsubsection*{Analysis of SOA}
This is the optimal algorithm in the realizable setting. The idea is same as that of \textit{Halving}. It also partitions the hypothesis class into two sets. One set consists of all those hypotheses which predict $0$ and other contains those which predict $1$ on the given point $x_t$. Then, unlike \textit{Halving}, it chooses prediction of the set which has larger $Ldim$ rather than the one with larger cardinality. When correct answer is revealed to the learner, it discards the set whose prediction are not same as the correct answer. 

The following Lemma proves the optimality of \textit{SOA}.
\begin{lemma}
	SOA enjoys the mistake bound $M_{SOA}(\mathcal{H}) \leq Ldim(\mathcal{H})$.\cite{Shalev-Shwartz:2012}
	\label{lemma:SOA_bound}
\end{lemma}

\begin{proof}
	It suffices to prove that whenever the algorithm makes a mistake, we have Ldim($V_{t+1}) \leq \text{Ldim}(V_t) - 1$. We prove this claim by assuming the contrary, that is, Ldim($V_{t+1}) = \text{Ldim}(V_t)$. If this holds true, then the definition of $p_t$ implies that Ldim($V_t^{(r)}) = \text{Ldim}(V_t)$ for both $r = 1$ and $r = 0$. But, then we can construct a shattered tree of depth Ldim$(V_t) + 1$ for the class $V_t$, which leads to the desired contradiction.
\end{proof}

Combining Lemma 1.1 and Lemma 2, we obtain:
\begin{corollary}
	Let $\mathcal{H}$ be any hypothesis class. Then, the standard
	optimal algorithm enjoys the mistake bound $M_{SOA}(\mathcal{H})$ = Ldim($\mathcal{H}$) and no other algorithm $A$ can have $M_A(\mathcal{H}) < \text{Ldim}(\mathcal{H})$ \cite{Shalev-Shwartz:2012}.
\end{corollary}

Table \ref{tab:realizable} summarizes all the available algorithms and their mistake bounds described above for finite hypothesis class and realizable case.

In all of the three algorithms described above, we have initialized $\eta$ by a fixed value which requires the value of $T$ to be known beforehand. But this is not necessary and can be fixed by using \textit{Doubling} trick which can be found in the literature \cite{Shalev-Shwartz:2012}.

\begin{table*}[t]
	\centering
	
	\scalebox{1.5}{
		\begin{tabular}{| c | c | c | c | }
			\hline
			\rowcolor{lightgray}
			Seq &  Algorithms &  Mistake Bound  &   Optimal Mistake Bound  \\ 
			\hline
			1.& Consistent   & $|\mathcal{H}|$        &\cellcolor[HTML]{EAE5EA} \\ 
			2.& Halving  	 & $\text{log}_2(|\mathcal{H}|)$   &\cellcolor[HTML]{EAE5EA} Ldim($\mathcal{H}$) \\
			3.& SOA          & Ldim($\mathcal{H}$)      &\cellcolor[HTML]{EAE5EA}  \\
			\hline
			
		\end{tabular}
	}
	\caption{Summary of mistake bounds of existing algorithms for finite hypothesis class and realizable case.}
	\label{tab:realizable}
\end{table*}

\subsection{\textbf{Finite hypothesis class and unrealizable case}} 
\label{subsec:WM}

In the unrealizable case, we do not assume that all the labels are generated from some $h^* \in \mathcal{H}$. Rather, we assume that there are some better hypotheses in $H$ which make lesser mistakes than others. We analyse the mistake bound with respect to these hypotheses and term this mistake bound as regret bound.

The key goal of all the algorithms described above for realizable case was to find the best hypothesis and then predict using that only. Realizability assumption gave us the liberty of discarding hypotheses which made mistake even once. We could discard the hypotheses because once they make a mistake; they can never be the target hypothesis. 

But in unrealizable case, there may not be such hypotheses; we cannot discard them just because they make  mistake in one or some rounds. We will have to keep track of mistake count of all the hypotheses and make our prediction based on the performance of each hypothesis so far. 

Even if $\mathcal{H}$ is finite, it can be arbitrarily large. If it is too large we cannot use all the hypotheses in consideration to make prediction. Therefore, based on this, there are two different algorithms for this setting.

\subsubsection{Prediction with expert advice: Weighted Majority (WM) Algorithm when $|\mathcal{H}|$ is finite}

 All the algorithms, described so far, were deterministic in their prediction. $Weighted\, Majority$ \cite{Littlestone1994,Shalev-Shwartz:2012} is a probabilistic algorithm which uses weighing mechanism to weigh each hypothesis based on their performance so far and treats their weights as a probability vector.

Let $\mathcal{H} = \{h_1,h_2,\cdots,h_d\}$. $Weighted\,Majority$ treats this class as set of experts which helps it predicting the answer to a given question $x_t$. This is given in Algorithm 5.
\begin{algorithm}[h]
	\caption{Weighted Majority: Learning with Expert Advice}
	\begin{algorithmic}
		\STATE \textbf{Input:} A finite hypothesis class $\mathcal{H}$ containing d experts. Number of rounds(input points) $T$
		\STATE \textbf{Initialize:} $\eta = \sqrt{2 \ln(d)/T}\, ;    \forall i \in [d], M_i^0 = 0$
		\FOR{$t = 1,2,\cdots$}
		\STATE receive $x_t$
		\STATE Receive expert advice $(h_1^t(x_t),$ 
		\STATE $h_2^t(x_t),\cdots,h_d^t(x_t))\in \{0,1\}^d$
		\STATE Define $w_i^{t-1} = \frac{e^{-\eta M_i^{t-1}}}{\sum_{j = 1}^d e^{-\eta M_j^{t-1}}} $
		\STATE Define $\hat{p_t} = \sum_{i:h_i^t(x_t) = 1} w_i^{t-1}$
		\STATE Predict $\hat{y_t} = 1$ with probability $\hat{p_t}$			
		\STATE receive true answer $y_t$
		\STATE update $M_i^t = M_i^{t-1} + \mathds{1}[h_i^t(x_t) \neq y_t]$ 
		\ENDFOR
	\end{algorithmic}
\end{algorithm}

\subsubsection*{Analysis of Weighted Majority}

Basically, $Weighted\,Majority$ assigns a weight $w \in [0,1]$ to each expert and keeps updating it based on the number of mistakes made by each expert upto current round.

When a point $x_t$ is received by the WM, it collects weights of all those experts which predict $1$ for this $x_t$ and accumulate these weights in a variable $p_t$. Then, it predict $1$ with probability $p_t$ (note that $p_t \in [0,1]$). When the true answer of $x_t$ is revealed to the algorithm, it updates the weights of each expert whichever made mistake on $x_t$.

The following theorem analyses the regret bound for the $Weighted\,Majority$ algorithm \cite{Shalev-Shwartz:2012}:
\begin{theorem}
	\label{thm:WM_bound}
	Weighted Majority satisfies the following :
	$$\sum_{t=1}^T \mathds{E}[\mathds{1}[\hat{y_t} \neq y_t]]- \min_{i \in [d]} \sum_{t=1}^T \mathds{1}[h_i^t(x_t) \neq y_t] \leq \sqrt{0.5 \ln(d) T}$$ 
\end{theorem}

\begin{proof}
	Refer to the Appendix A.
\end{proof}

This work is based on the construction of an algorithm to improve this bound in realizable case.\\

\subsubsection{Expert algorithm:  When $|\mathcal{H}|$ is allowed to be infinite but Ldim($\mathcal{H}$) is finite}\cite{Shalev-Shwartz:2012}

When $\mathcal{H}$ is allowed to be infinite, we cannot use each hypothesis in each round for deciding the prediction. But we assume that Ldim($\mathcal{H}$) is finite so that we can use \textit{SOA} in some way.

Since we are assuming that $\mathcal{H}$ can be infinite and we cannot use each hypothesis in each round, the challenge therefore is to define a set of experts that on one hand is not excessively large while on the other hand contains an expert that gives accurate prediction.

Here, basic idea is to simulate each expert by running \textit{SOA} algorithm on a small sub-sequence of points. We define an expert for each sequence of length $L < $= Ldim($\mathcal{H})$ and then use that constructed set of experts on that sub-sequence.
Further details about this algorithm can be found in literature \cite{Shalev-Shwartz:2012}.

Table \ref{tab:unrealizable} summarizes all the available algorithms and their regret bounds described above for finite Ldim hypothesis class in unrealizable case.

\begin{table*}[h!]
	\centering
	\scalebox{1.3}{
		\begin{tabular}{| c | c | c | c | }
			\hline
			\rowcolor{lightgray}
			Seq &  Algorithms &  Regret Bound  &   Optimal Regret Bound  \\ 
			\hline
			1.& \pbox{20cm}{Weighted Majority\\ when $|\mathcal{H}| < \infty$\\ Ldim($\mathcal{H}) < \infty$ }  &  $\sqrt{(0.5\, \ln(\mathcal{H})\, T\,)}$         &\cellcolor[HTML]{EAE5EA} O($\sqrt{\text{Ldim}(\mathcal{H})\,T}$)\\  \hline
			2.& \pbox{20cm}{Experts when $|\mathcal{H}|$ \\ is allowed to be $\infty$,\\Ldim$|\mathcal{H}| < \infty$}   	 & \pbox{20cm}{$\text{Ldim}(\mathcal{H}) + $\\$\sqrt{0.5\, \text{Ldim}(\mathcal{H})\, T\, \ln(T)}$ }   &\cellcolor[HTML]{EAE5EA} Same as above \\
			
			\hline
		\end{tabular}
	}
	\caption{Summary of regret bounds of existing algorithms for finite Ldim hypothesis class in unrealizable case.}
	\label{tab:unrealizable}
\end{table*}

\subsection{\textbf{Infinite hypothesis class with realizable and unrealizable cases}}

This case is out of the scope of this work. However there are some algorithms like \textit{ Perceptron} and \textit{Winnow} in this setting which can be found in the literature  \cite{blum1998line,cesa2006prediction,rakhlin2012statistical,Shalev-Shwartz:2012}.

\section{Proposed methodologies and Theoretical results}

\subsection{Problem Statement and Objectives}

As mentioned earlier, in the realizable case, we assume that all the labels are generated from some $h^* \in \mathcal{H}$. i.e. these $h^*$ does not make any mistake on the entire sequence of $T$ points. In the unrealizable (agnostic) case, we no longer assume this. However, we assume that there are some better hypotheses in $\mathcal{H}$ which make lesser mistakes than others.
\par In the literature, there are methods which are developed for both of these cases exclusively. It's true that the methods which are developed for unrealizable case will also work for realizable case. But existing methods do not have much better bound if the input sequence is happened to be realizable in the end.

Since, in practice, we usually do not have any definite prior information about the realizability or unrealizability of the input sequence, we can not be sure about which algorithm to use over a given input sequence. Whether we should use $SOA$ which is optimal in realizable case or $Weighted \; Majority$ which is optimal in unrealizable case in standard settings.  
Also, we do not have liberty to reiterate over the input sequence in online learning framework and hence we can not try out existing algorithms one after the other to figure out which one we should have used!

\par \textit{Particularly, can we do somewhat better if we have some prior information about the realizability of input sequence?} Especially, the scenario where we observe that the input sequence is \textit{likely} to be realizable. We can not use $SOA$ because it will fail if sequence is found to be unrealizable. Also, It would not be efficient to use $Weighted \, Majority$ because it does not perform that better if sequence is found to be realizable. It does not have any fixed and better mistake bound in realizable case. Therefore, we would like to devise some methods which have fixed and better mistake bound in realizable case. That is, they should perform extremely well if sequence is found to be realizable and do not perform that bad if it happens to be unrealizable. 
\par Thus, we begin with the following objectives:
\begin{enumerate}[(i)]
	\item Devise some methods for the finite hypothesis class which perform extremely well if input sequence is found to be realizable and do not perform that bad if input sequence is found to be unrealizable. In other words, we would like to develop some methods which improve the mistake bound in realizable case greatly while lose slightly in regret bound in unrealizable case.
	
	\item We would also like to get hold of the best hypotheses at the end of the sequence. This might be required for various reasons. For example, an application of learning algorithm might be just to find the best hypotheses in the class for a input sequence.
\end{enumerate}

\subsection{Approach}

If we observe the algorithms described in the section \ref{subsec:FHRC} for the finite hypothesis class and realizable case, the mistake bound does not depend on the length $T$ of the input sequence. It depends only on the size of the hypothesis class $|\mathcal{H}|$.
Further, the $Weighted \; Majority$ algorithm described in section \ref{subsec:WM} enjoys the regret bound which depends on both $T$ and $|\mathcal{H}|$, \textit{no matter} the input sequence is realizable or unrealizable.
This is what that drives the idea of devising the proposed methods. Therefore, the proposed algorithms combine the approaches of existing algorithms presented in section \ref{subsec:FHRC} and \ref{subsec:WM}. The following section presents the proposed algorithms. 

These new algorithms couple the $Weighted \, Majority$ algorithm with each of the algorithm described in section \ref{subsec:FHRC}, namely, $Consistent,\, Halving$ and $SOA$. They improves the bound quite much in the realizable case while do not lose that much in the unrealizable case.

For all the proposed algorithms, we make following assumptions:
\begin{enumerate}[(i)]
	\item $|\mathcal{H}|$ is finite and hence Ldim($\mathcal{H}$) is finite.
	\item Ldim$(\mathcal{H}) << T$
\end{enumerate}

As already mentioned, the proposed algorithms couple the $ Weighted\,Majority$ algorithm with each of the algorithm described in section \ref{subsec:FHRC}, namely, $Consistent$, $Halving$,  and $SOA$. Therefore, we name them $WM\_Consistent$, $WM\_Halving$ and $WM\_SOA$ respectively. Out of the three proposed algorithms, $WM\_SOA$ is the best algorithm in terms of mistake and regret bounds. It is described in Algorithm \ref{alg:WM_SOA}. 

\subsection{Proposed Methodologies}
All the proposed methodologies are described in detail as follows.

\begin{enumerate}[(i)]
	\item $\mathbf{Weighted \,Majority\, with\, Consistent}$
	
	The following algorithm combines the $ Weighted\, Majority$ and $ Consistent $ algorithms. This is the most basic algorithm out of the proposed three algorithms. Therefore, it does not have much better bound. However, it is still better than the $ Weighted\, Majority$ in realizable case.
	The proposed algorithm is given in Algorithm \ref{alg:WM_Consistent}.
	\begin{algorithm}[h]
		\caption{$WM\_Consistent$}
		\begin{algorithmic}
			\STATE \textbf{Input:} A finite hypothesis class $\mathcal{H}$ containing d experts. i.e $|\mathcal{H}|=d$. Number of rounds (input points) $T$
			\STATE \textbf{Initialize:} $\eta = \sqrt{2 \ln(d)/T}\,; \forall i \in [d], M_i^0 = 0$   
			\STATE \textbf{Initialize:} $V_1 = \mathcal{H}$
			\FOR{$t = 1,2,\cdots$}
			\STATE Receive $x_t$
			\IF {($V_t \, \text{is not empty}$)}
			\STATE choose any $h \in V_t$
			\STATE predict $p_t = h(x_t)$
			\STATE Receive true answer $y_t = h^*(x_t)$
			\STATE Update $V_{t+1} = \{ h \in V_t : (h(x_t) = y_t)\}$
			\STATE Update $M_i^t = M_i^{t-1} + \mathds{1}[h_i^t(x_t) \neq y_t]$  
			\ELSE
			\STATE Receive expert advice $(h_1^t(x_t),$ 
			\STATE $h_2^t(x_t),\cdots,h_d^t(x_t))\in \{0,1\}^d$
			\STATE Define $w_i^{t-1} = \frac{e^{-\eta M_i^{t-1}}}{\sum_{j = 1}^d e^{-\eta M_j^{t-1}}} $
			\STATE Define $\hat{p_t} = \sum_{i:h_i^t(x_t) = 1} w_i^{t-1}$
			\STATE Predict $\hat{y_t} = 1$ with probability $\hat{p_t}$			
			\STATE Receive true answer $y_t$
			\STATE Update $M_i^t = M_i^{t-1} + \mathds{1}[h_i^t(x_t) \neq y_t]$
			\ENDIF
			\ENDFOR
		\end{algorithmic}
		\label{alg:WM_Consistent}
	\end{algorithm}
		
	\subsubsection*{ Analysis of WM{\_}Consistent} 
	
	Let $T$ be the number of input points and $|\mathcal{H}|$ be the size of the hypothesis class. We assume that $|\mathcal{H}|$ is finite.
	The basic idea here is to use $ Consistent $ algorithm in the beginning of the sequence. Since $Consistent$ algorithm discards all the hypothesis which predict wrong on any input point, we keep using the predictions of $Consistent$ until its current hypothesis class $V_t$ becomes empty. Depending on whether the sequence is realizable or not, there are following two cases to be considered:
	
	\subsubsection*{Case 1: When input sequence is realizable by $\mathcal{H}$}
	
	From  section \ref{subsec:FHRC}, Algorithm \ref{alg:Consistent}, we know that mistake bound of $Consistent$ algorithm is ($|\mathcal{H}|-1$). Therefore in the worst case, instance of $Consistent$ will make at most ($|\mathcal{H}|-1$) mistakes. Thereafter, it will not make any mistake.
	
	Because of realizability assumption, even after processing $|\mathcal{H}|$ points, $V_t$ will not be empty and will contain at least one hypothesis. These are those hypotheses which realize the input sequence. There can be more than one such hypotheses.  Hence, $Consistent$ portion will be continuing for the rest of the ($T-|\mathcal{H}|-1$) points and will not make any mistake further. Thus, in the realizable case the mistake bound of proposed algorithm is following :
	
	\begin{equation}
	\label{eq:Consistent}
	M_{WM\_Consistent}(\mathcal{H}, T) \leq |\mathcal{H}-1|
	\end{equation}

	\subsubsection*{Case 2: When input sequence is not realizable by $\mathcal{H}$}
	
	When the input sequence is not realizable by the hypothesis class, the $Consistent$ portion of the algorithm will make at most $|\mathcal{H}|$ mistakes . Hence, in the worst case, till or before $|\mathcal{H}|$ points, $V_t$ will become empty and hereafter, $Weighted \, Majority$ will start predicting the $\hat{y_t}$. 
	
	Note that while $Consistent$ portion remains active, whenever algorithm makes any mistake, we simultaneously update the mistake count of each hypothesis. Once $V_t$ becomes empty and $Weighted\,Majority$ becomes active,  $Weighted\,Majority$ receives the updated mistake count list of each hypothesis on the points seen so far. This gives $Weighted\,Majority$ the same state as if it was being used for predicting $\hat{y_t}$ from the very beginning. 
	
	This can be guaranteed that the point where $ V_t $ becomes empty, mistake count of each hypothesis is at least 1. If not so, algorithm would have been continued using $ Consistent $ with the hypothesis whose mistake count is $ 0 $. This way, it helps $Weighted\,Majority$ portion of  $WM\_Consistent$ to distribute the initial weights of each hypothesis based on its mistake count and weigh the advice of each hypothesis accordingly.
	
	From the Theorem \ref{thm:WM_bound}, we know that $Weighted\,Majority$ enjoys the following expected regret bound on a given input sequence of \textit{length $T$ }
	
	\begin{equation}
	\label{eq:WM_Consistent_eq}
	\begin{split}
	\sum_{t=1}^T \mathds{E}[\mathds{1}[\hat{y_t} \neq y_t]]- \min_{i \in [d]} \sum_{t=1}^T \mathds{1}[h_i^t(x_t) \neq y_t] \\
	 \leq \sqrt{0.5\, \ln(|\mathcal{H}|)\, T}
	 \end{split}
	\end{equation}
	
	After analysing the realizable and unrealizable cases above separately, we present the following theorem which presents the regret bound of the proposed $WM\_Consistent$ algorithm.
	\begin{theorem}
		Algorithm \ref{alg:WM_Consistent} enjoys the following mistake and expected regret bound 
		\begin{enumerate}
			\item In realizable case, mistake bound:
			$$M_{WM\_Consistent}(\mathcal{H}, T) \leq |\mathcal{H}-1|$$
			
			\item In unrealizable case, Weighted Majority will receive the input sequence of length at least  $(T-|\mathcal{H}|)$. Hence, the expected regret bound becomes following:
			$$\sum_{t=1}^T \mathds{E}[\mathbf{1}[\hat{y_t} \neq y_t]]- \min_{h \in \mathcal{H}} \sum_{t=1}^T \mathbf{1}[h(x_t) \neq y_t] $$  
			$$ \leq 
			\sqrt{0.5\,\ln(|\mathcal{H}|) \, (T-|\mathcal{H}|)} +  |\mathcal{H}|$$
		\end{enumerate}
	\end{theorem} 	
	
	\begin{proof}
		Since the input sequence either be realizable or unrealizable, we have discussed the bounds in both of the cases separately  and equation \ref{eq:Consistent} and \ref{eq:WM_Consistent_eq} clearly establish the above required bound.

	\end{proof}

	In the following subsection, we couple the $ Weighted \, Majority $ with $ Halving $ algorithm described in the Algorithm \ref{alg:halving}.\\
	
	\item $\mathbf{Weighted\, Majority\, with\, Halving}$ 
	
	The proposed algorithm is given in Algorithm \ref{alg:WM_Halving}.
	\begin{algorithm}[h]
		\caption{$WM\_Halving$}
		\begin{algorithmic}
			\STATE \textbf{Input:} A finite hypothesis class $\mathcal{H}$ containing d experts. Number of rounds (input points) $T$
			\STATE \textbf{Initialize:} $\eta = \sqrt{2 \ln(d)/T}\,; \forall i \in [d], M_i^0 = 0$   
			\STATE \textbf{Initialize:} $V_1 = \mathcal{H} $
			\FOR{$t = 1,2,\cdots$}
			\STATE Receive $x_t$
			\IF {($V_t \, \text{is not empty}$)}
			\STATE Predict $p_t = \operatorname{arg\,\max}_{r \in \{0,1\}} |\{ h \in V_t : (h(x_t) = r)\}|$
			\\(in case of a tie predict $p_t = 1$)
			\STATE Receive true answer $y_t = h^*(x_t)$
			\STATE Update $V_{t+1} = \{ h \in V_t : (h(x_t) = y_t)\}$
			\STATE Update $M_i^t = M_i^{t-1} + \mathds{1}[h_i^t(x_t) \neq y_t]$  
			\ELSE
			\STATE Receive expert advice $(h_1^t(x_t),$ 
			\STATE $h_2^t(x_t),\cdots,h_d^t(x_t))\in \{0,1\}^d$
			\STATE Define $w_i^{t-1} = \frac{e^{-\eta M_i^{t-1}}}{\sum_{j = 1}^d e^{-\eta M_j^{t-1}}} $
			\STATE Define $\hat{p_t} = \sum_{i:h_i^t(x_t) = 1} w_i^{t-1}$
			\STATE Predict $\hat{y_t} = 1$ with probability $\hat{p_t}$			
			\STATE Receive true answer $y_t$
			\STATE Update $M_i^t = M_i^{t-1} + \mathds{1}[h_i^t(x_t) \neq y_t]$
			\ENDIF
			\ENDFOR
		\end{algorithmic}
		\label{alg:WM_Halving}
	\end{algorithm}

	\subsubsection*{\\Analysis of WM{\_}Halving} 
	
	Let $T$ be the number of input points and $|\mathcal{H}|$ be the size of the hypothesis class. We assume that $|\mathcal{H}|$ is finite.
	The basic idea here is to use $Halving$ algorithm in the beginning of the sequence. Since $Halving$ algorithm discards at least half of the hypothesis once it makes any mistake, we keep using the predictions of $Halving$ until its current hypothesis class $V_t$ becomes empty. Depending on whether the sequence is realizable or not, there are following two cases to be considered:
	
	\subsubsection*{Case 1 : When input sequence is realizable by $\mathbf{ \mathcal{H}}$}
	From  section \ref{subsec:FHRC}, Algorithm \ref{alg:halving}, we know that mistake bound of Halving algorithm is $\log_2(|\mathcal{H}|)$. Therefore, in the worst case, instance of Halving will make at most $\log_2(|\mathcal{H}|)$ mistakes. Thereafter, it will not make any mistake further. 
	
	Since in this case, even after processing $\log_2(|\mathcal{H}|)$ points, $V_t$ will not be empty and will contain at least one hypothesis. These are those hypotheses which realize the input sequence.  Hence, $ Halving $ portion will be continuing prediction for the rest of the ($T-\log_2(|\mathcal{H}|)$) points and will not make any mistake further. Thus, in the realizable case, the mistake bound of the proposed algorithm is following : 
	
	\begin{equation}
	\label{eq:WM_Halving}
	M_{WM\_Halving}(\mathcal{H}, T) \leq |\log_2(|\mathcal{H}|)
	\end{equation}

	\subsubsection*{Case 2 : When input sequence is not realizable by $\mathbf{ \mathcal{H}}$}
	When the input sequence is not realizable by the hypothesis class, the $Halving$ portion of the algorithm will make at most $\log_2(|\mathcal{H}|)$ mistakes. In the worst case, till or before $\log_2(|\mathcal{H}|)$ points, $V_t$ will become empty and hereafter, $ Weighted\, Majority $ will start predicting the $\hat{y_t}$. 
	
	Note that while $Halving$ portion remains active, whenever algorithm makes any mistake, we simultaneously update the mistake count of each hypothesis. Once $V_t$ becomes empty and $Weighted\,Majority$ becomes active,  $Weighted\,Majority$ receives the updated mistake count list of each hypothesis on the points seen so far. This gives $Weighted\,Majority$ the same state as if it was being used for predicting $\hat{y_t}$ from the very beginning. 
	
	This can be guaranteed that the point where $ V_t $ becomes empty, mistake count of each hypothesis is at least 1. If not so, algorithm would have been continued using $ Halving $ with the hypothesis whose mistake count is $ 0 $. This way, it helps $Weighted\,Majority$ portion of  $WM\_Halving$ to distribute the initial weights of each hypothesis based on its mistake count and weigh the advice of each hypothesis accordingly.
	
	From the Theorem \ref{thm:WM_bound}, we know that $Weighted\,Majority$ enjoys the following expected regret bound on a given input sequence of \textit{length $T$ } 
	
	\begin{equation}
	\label{eq:WM_Halving_eq}
	\begin{split}
	\sum_{t=1}^T \mathds{E}[\mathds{1}[\hat{y_t} \neq y_t]]- \min_{i \in [d]} \sum_{t=1}^T \mathds{1}[h_i^t(x_t) \neq y_t] \\
	\leq 
	\sqrt{0.5\, \ln(\mathcal{H})\, T}
	\end{split}
	\end{equation}
	
	After analysing the realizable and unrealizable cases above separately, we present the following theorem which presents the regret bound of the proposed $WM\_Halving$ algorithm.
	\begin{theorem}
		Algorithm 9 enjoys the following expected regret bound 
		\begin{enumerate}
			\item In realizable case, mistake bound:
			$$M_{WM\_Halving}(\mathcal{H}, T) \leq \log_2(|\mathcal{H}|)$$
			
			\item In unrealizable case, Weighted Majority will receive the input sequence of length at least $(T-\log_2(|\mathcal{H}|)$. Hence, the expected regret bound becomes following:
			
			$$\sum_{t=1}^T \mathds{E}[\mathbf{1}[\hat{y_t} \neq y_t]]- \min_{h \in \mathcal{H}} \sum_{t=1}^T \mathbf{1}[h(x_t) \neq y_t] $$
			
			$$\leq \sqrt{0.5\,\ln(|\mathcal{H}|) \, (T-\log_2|\mathcal{H}|)} + \text{log}_2(|\mathcal{H}|)$$
			
		\end{enumerate}
		\label{thm:WM_Halving} 
	\end{theorem} 	
	
	\begin{proof}
		Since the input sequence either be realizable or unrealizable, we have discussed the bounds in both of the cases separately  and equation \ref{eq:WM_Halving} and \ref{eq:WM_Halving_eq} clearly establish the above required bound.\\
	\end{proof}

	\item $\mathbf{Weighted\, Majority\, with\, SOA}$
	
	This algorithm enjoys the best expected regret bound. This comes from the fact that unlike $ Halving $, it uses \textit{Ldim} of the two different sets $ V_t^{(0)}$ and $ V_t^{(1)}$ and uses the set for prediction which has the larger \textit{Ldim}.
	The proposed algorithm is given in the Algorithm \ref{alg:WM_SOA}.
	\begin{algorithm}[h]
		\caption{$WM\_SOA$}
		\begin{algorithmic}
			\STATE \textbf{Input:} A finite hypothesis class $\mathcal{H}$ containing d experts. Number of rounds (input points) $T$
			\STATE \textbf{Initialize:} $\eta = \sqrt{2 \ln(d)/T}\,; \forall i \in [d], M_i^0 = 0$   
			\STATE \textbf{Initialize:} $V_1 = \mathcal{H} $
			\FOR{$t = 1,2,\cdots$}
			\STATE Receive $x_t$
			\IF {($V_t \, \text{is not empty}$)}
			\STATE for $r \in \{0,1\}$ let $ V_t^{(r)} = \{ h \in V_t : h(x_t) = r\}$ 
			\STATE predict $p_t = \operatorname{arg\,\max}_{r \in \{0,1\}} \text{Ldim}(V_t^{(r)})$
			\\(in case of a tie predict $p_t = 1$)
			\STATE Receive true answer $y_t = h^*(x_t)$
			\STATE Update $V_{t+1} = \{ h \in V_t : (h(x_t) = y_t)\}$
			\STATE Update $M_i^t = M_i^{t-1} + \mathds{1}[h_i^t(x_t) \neq y_t]$  
			\ELSE
			\STATE Receive expert advice $(h_1^t(x_t),$ 
			\STATE $h_2^t(x_t),\cdots,h_d^t(x_t))\in \{0,1\}^d$
			\STATE Define $w_i^{t-1} = \frac{e^{-\eta M_i^{t-1}}}{\sum_{j = 1}^d e^{-\eta M_j^{t-1}}} $
			\STATE Define $\hat{p_t} = \sum_{i:h_i^t(x_t) = 1} w_i^{t-1}$
			\STATE Predict $\hat{y_t} = 1$ with probability $\hat{p_t}$			
			\STATE Receive true answer $y_t$
			\STATE Update $M_i^t = M_i^{t-1} + \mathds{1}[h_i^t(x_t) \neq y_t]$
			\ENDIF
			\ENDFOR
		\end{algorithmic}
		\label{alg:WM_SOA}
	\end{algorithm}

	\subsubsection*{Analysis of WM{\_}SOA}

	Let $T$ be the number of input points and $|\mathcal{H}|$ be the size of the hypothesis class. We assume that $|\mathcal{H}|$ is finite.
	The basic idea here is to use $ SOA $ algorithm in the beginning of the sequence. From Algorithm \ref{alg:soa}, we know that $ SOA $ algorithm discards at least as many as Ldim($\mathcal{H}$) hypotheses whenever it makes a mistake (From the Lemma 1.2, we also know that Ldim$(\mathcal{H}) \leq \text{log}_2(|\mathcal{H}|)$. We keep using the predictions of $ SOA $ until its current hypothesis class $V_t$ becomes empty. Depending on whether the sequence is realizable or not, there are following two cases to be considered:\\

	\subsubsection*{Case 1: When input sequence is realizable by $\mathbf{ \mathcal{H}}$}
	From  Lemma \ref{lemma:SOA_bound}, we know that mistake bound of $ SOA $ algorithm is Ldim$(\mathcal{H})$. Therefore, in the worst case, instance of $ SOA $ will make at most Ldim$(\mathcal{H})$ mistakes. Thereafter, it will not make any mistake. 
	
	Since in this case, even after processing Ldim$(\mathcal{H})$ points, $V_t$ will not be empty and will contain at least one hypothesis. These are those hypotheses which realize the input sequence.  Hence, $ SOA $ portion will be continuing prediction for the rest of the ($T-\text{Ldim}(\mathcal{H})$) points and will not make any mistake further. Thus, in the realizable case the mistake bound of proposed algorithm is following :
	
	\begin{equation}
	\label{eq:WM_SOA}
	M_{WM\_SOA}(\mathcal{H}, T) \leq Ldim(\mathcal{H})
	\end{equation}

	%
	
	\subsubsection*{Case 2 : When input sequence is not realizable by $\mathbf{ \mathcal{H}}$}
	
	When the input sequence is not realizable by the hypothesis class, the $SOA$ portion of the algorithm will make at most as many as Ldim$(\mathcal{H})$ mistakes. In the worst case, till or before Ldim$(\mathcal{H})$ points, $V_t$ will become empty and hereafter, $ Weighted\, Majority $ will start predicting the $\hat{y_t}$. 
	
	Note that while $SOA$ portion remains active, whenever algorithm makes any mistake, we simultaneously update the mistake count of each hypothesis. Once $V_t$ becomes empty and $Weighted\,Majority$ becomes active,  $Weighted\,Majority$ receives the updated mistake count list of each hypothesis on the points seen so far. This gives $Weighted\,Majority$ the same state as if it was being used for predicting $\hat{y_t}$ from the very beginning. 
	
	This can be guaranteed that the point where $ V_t $ becomes empty, mistake count of each hypothesis is at least 1. If not so, algorithm would have been continued using $ SOA $ with the hypothesis whose mistake count is $ 0 $. This way, it helps $Weighted\,Majority$ portion of  $WM\_SOA$ to distribute the initial weights of each hypothesis based on its mistake count and weigh the advice of each hypothesis accordingly.
	
	From the Theorem \ref{thm:WM_bound}, we know that $Weighted\,Majority$ enjoys the following expected regret bound on a given input sequence of \textit{length} $T$ 
	
	\begin{equation}
	\label{eq:WM_SOA_eq}
	\begin{split}
	\sum_{t=1}^T \mathds{E}[\mathds{1}[\hat{y_t} \neq y_t]]- \min_{i \in [d]} \sum_{t=1}^T \mathds{1}[h_i^t(x_t) \neq y_t] \\
	\leq \sqrt{0.5\, \ln(|\mathcal{H}|)\, T}
	\end{split}
	\end{equation}
	
	After analysing the realizable and unrealizable cases above separately, we present the following theorem which presents the regret bound of the proposed $WM{\_}SOA$ algorithm.
	\begin{theorem}
		Algorithm 9 enjoys the following expected regret bound 
		\begin{enumerate}[(a)]
			\item In realizable case:
			$$M_{WM\_SOA}(\mathcal{H}, T) \leq Ldim(\mathcal{H})$$
			
			\item  In unrealizable case, Weighted Majority will receive the input sequence of length at least $(T-\text{Ldim}(\mathcal{H}))$. Hence, the expected regret bound becomes following:
			$$\sum_{t=1}^T \mathds{E}[\mathds{1}[\hat{y_t} \neq y_t]]- \min_{h \in \mathcal{H}} \sum_{t=1}^T \mathds{1}[h(x_t) \neq y_t]  $$
			$$\leq  \sqrt{0.5\,\ln(|\mathcal{H}|) \, (T-\text{Ldim}(\mathcal{H}))} + \text{Ldim}(\mathcal{H})$$
		\end{enumerate}
	\end{theorem} 	
	
	\begin{proof}
		Since the input sequence either be realizable or unrealizable, we have discussed the bounds in both of the cases separately  and equation \ref{eq:WM_SOA} and \ref{eq:WM_SOA_eq} clearly establish the above required bound.
	\end{proof}
	\vspace{0.2cm}
	
	Our second objective was to get hold of best functions at the end of the sequence. It can be seen that it has also been achieved in all three proposed algorithms by keeping mistake count of each function. The functions which have least number of mistake count at the end; are best functions. Note that, in realizable case, the functions which have mistake count equals to 0 are best functions.
	\vspace{0.1cm}
	
	In this section, we described several methods for the finite hypothesis class and both realizable and unrealizable case settings. In particular, we presented three methods to improve the mistake bound for these setting using \textit{Weighted Majority} with other methods picked up from finite hypothesis class and realizable case setting. Proposed methods do not improve the regret bound in general but they greatly improve the mistake bound for the realizable case while losing a little in unrealizable case.
\end{enumerate}

\subsection*{\textbf{Contributions}}
Firstly, table \ref{tab:Contributions} summarize the mistake and regret bounds of proposed algorithms and compare them with that of existing ones.

\begin{table*}[h!]
	\centering
	\resizebox{1.2\textwidth}{!}{\begin{minipage}{\textwidth}
			\scalebox{1.1}{
				\begin{tabular}{| c | c | c | c | c | }
					\hline
					\rowcolor{lightgray}
					Seq
					& Algorithms
					& \pbox{20cm}{Mistake Bound \\ (Realizable case)} 
					& \pbox{20cm}{Regret Bound \\ (Unrealizable case)} 
					& \pbox{20cm}{Optimal \\ Regret Bound \\ (Unrealizable \\ case)}  \\ 
					\hline
					1.
					& \pbox{20cm}{Weighted Majority} 
					& $\sqrt{0.5\, \ln(\mathcal{H})\, T\,}$  
					& $\sqrt{0.5\, \ln(\mathcal{H})\, T\,}$        
					&  ($\sqrt{\text{Ldim}(\mathcal{H})\,T}$) \\  \hline
					
					2.
					& WM{\_}Consistent	
					& $|\mathcal{H}|-1$
					& \pbox{20cm}{$|\mathcal{H}| + \sqrt{0.5\, \ln(|\mathcal{H}|)\, (T-|\mathcal{H}|)}$}
					& Same as above \\ \hline
					
					3.
					& WM{\_}Halving	
					& log$_2(\mathcal{H})$
					& \pbox{20cm}{ $\log_2(\mathcal{H}) + \sqrt{0.5\, \ln(|\mathcal{H}|)\, (T- \log_2(|\mathcal{H}|))}$}
					& Same as above \\ \hline
					
					4.
					& WM{\_}SOA
					& Ldim$(\mathcal{H})$	
					& \pbox{20cm}{$\text{Ldim}(\mathcal{H}) + \sqrt{0.5\, \ln(|\mathcal{H}|)\, (T- \text{Ldim}(\mathcal{H}))}$} 
					& Same as above \\ \hline
				\end{tabular}
			}
		\end{minipage} }
		
		\caption{Comparison of mistake and regret bounds of existing and proposed algorithms for realizable and unrealizable cases under finite hypothesis class .}
		\label{tab:Contributions}
	\end{table*}

	It can be observed that the mistake and regret bounds of $Weighted\, Majority$ are same. Proposed algorithms reduces the mistake bound by a large factor while lose very little in regret bound.
	
	For example,  In case of WM{\_}SOA, the mistake bound is reduced from  $\sqrt{0.5\, \ln(|\mathcal{H}|)\, T}$ to Ldim$(\mathcal{H})$ which is no more even a function of $T$. On the other hand, the regret bound is increased from  $\sqrt{0.5\, \ln(|\mathcal{H}|)\, T}$ to just $\text{Ldim}(\mathcal{H}) + \sqrt{0.5\, \ln(|\mathcal{H}|)\, (T- \text{Ldim}(\mathcal{H}))}$ with and additional quantity of $\text{Ldim}(\mathcal{H})$. Note that, with the assumption mentioned earlier that  Ldim$(\mathcal{H}) << T$, this additional quantity does not incur much loss to the regret bound. Also note that, at the same time, the $\sqrt{T}$ term has also reduced to $\sqrt{(T-\text{Ldim}(\mathcal{H}))}$.

\section{Simulations}

This section gives supplementary backing to the results claimed in proposed methodologies in section three. Since implementation of all the three proposed algorithms is very much similar to that of one another and Algorithm \ref{alg:WM_Halving} is easy to implement, we are going to simulate only this WM{\_}Halving algorithm.

Although, the simulations presented in this chapter are neither necessary nor sufficient to prove bounds for any proposed algorithm.  In fact, the performance of an online learning algorithms is analysed solely by theoretical proofs for mistake and regret bounds.
Still, the simulations are presented in this chapter with the following objectives:
\begin{itemize}
	\item To show that the proposed algorithms indeed conform to the given regret bound.
	\item To present a simple implementation scenario for the online learning framework.
\end{itemize} 

For implementing proposed algorithms, we need the following : 
\begin{enumerate}[(i)]
	\item Input sequence of $T$ labelled points (realizable or unrealizable by $\mathcal{H}$ based on the scenario under consideration)
	\item Finite hypothesis class $\mathcal{H}$	
\end{enumerate}

\subsection{Construction of hypothesis class $\mathcal{H}$ and input sequence $S$:}

This section describes the process of constructing hypothesis class and input sequence. This constructed input sequence may be realizable or unrealizable depending on the hypothesis class $\mathcal{H}$. 

For implementation purpose, we need to generate both types of input sequences. To analyse the realizable case, we need to generate it in such a way that $\exists \, h^* \in \mathcal{H}$ such that
$$h^*(x_t) = y_t\, ; \; \forall t \in \left[T\right]$$

We generate points $(x_t,\,y_t)$ of sequence S  satisfying the following : 
$$ x_t \in [-T/2+1,\, T/2 ] $$
and 
$$y_t \in \{0,1\}$$

\begin{enumerate}[(i)]
{\item \textbf{Generating hypothesis class $\mathcal{H}$: }}

Since we have $T$ many points and all $ x_t \in [-T/2+1,\, T/2 ] $, the following is the way adopted for generating functions such that neither all of them make correct predictions on all the $ x_t $'s (except the realizable case) nor all the functions make mistakes on all the $ x_t $'s.

\begin{equation}
\label{eq:def_hi}
h_i(x_t)=\left\{ \begin{array}{cl}
0 & \textrm{if } x_t \leq i \\
1 & \textrm{otherwise} \\
\end{array}\right.
\end{equation}

That is, $ i^{th}$ function will assign label $0$ to all the $ x_t$'s which are $ \leq i $ and assigns $ 1 $ to all those $ x_t $'s which are greater than $i$. Naturally, $i$ is an integer and should lie in $\left[1,d\right]$, where $d$ is the number of functions in hypothesis class (i.e. $d = |\mathcal{H}|$) and for this example $d \leq T$

So our hypothesis class becomes the set of these $h_i$'s:
\begin{equation}
\label{eq:H_generate}
\mathcal{H} = \{ h_i : i \in \left[1,d\right] \}
\end{equation}
where  $h_i$ is defined in eq \ref{eq:def_hi}.\\

\item \textbf{Generating input sequence $S$ labeled by some $h$:}  	

We will generate realizable and unrealizable sequences $S$ of length $T$ in the following way.

\begin{enumerate} [(a.)]
	\item \textit{Constructing Realizable Sequence:}
For realizable sequence, we have generated the labels of $x_t$'s of sequence $S$ using a hypothesis $h^*$ defined below:

\begin{equation}
\label{eq:realizable_generate}
h_0(x_t)=\left\{ \begin{array}{cl}
0 & \textrm{if } x_t \leq 0 \\
1 & \textrm{otherwise} \\
\end{array}\right.
\end{equation}

That is, it assigns label $0$ to all non positive points and label $1$ to all positive points. Table \ref{tab:realizable_seq} demonstrate an example of a sequence generated by the function in equation \ref{eq:realizable_generate} with $T = 8$.

\begin{table}[h!] 
	\centering
	\label{tab:sequence}
	\captionsetup{justification=centering,margin=1cm}
	\scalebox{1.3}{
	\begin{tabular}[t]{| c | c | c | }
		\hline
		\rowcolor{lightgray}
		$t$ &  $x_t$ & $y_t$ \\ 
		\hline
		1.& $-3$ & $0$ \\  \hline
		2.& $-2$ & $0$ \\  \hline
		3.& $-1$ & $0$ \\  \hline
		4.& $0$ & $0$ \\  \hline
		5.& $1$ & $1$ \\  \hline
		6.& $2$ & $1$ \\  \hline
		7.& $3$ & $1$ \\  \hline
		8.& $4$ & $1$ \\  \hline
	\end{tabular}
	}
	\caption{Example sequence of length 8 which is realizable by the hypothesis class constructed according to function definition given by equation \ref{eq:realizable_generate}.}
	\label{tab:realizable_seq}
\end{table}

So the sequence is 

\begin{equation}
    \label{eq:realizable_seq}
    \begin{split}
        S = \{(-3,0), (-2,0), (-1,0), (0,0), (1,1), \\ (2,1), (3,1), (4,1)\}
    \end{split}
\end{equation} 

Note that $h_0$ defined in eq (\ref{eq:realizable_generate}) realizes this sequence.

Table \ref{tab:unrealizable_performance} shows the prediction of each function in $\mathcal{H}$ constructed by equation \ref{eq:H_generate} with $d = 5$ over the sequence defined in eq (\ref{eq:realizable_seq}).

In this table, $M(h_i)$ denotes the total number of mistakes made by the hypothesis $i$ on the entire sequence $S$. (Red coloured number shows the wrong prediction.)

It can be seen that $h_0$ does not make any mistake. Hence, it realizes the sequence S.

\item \textit{Constructing Unrealizable Sequence:}
To obtain an unrealizable sequence, labels of $x_t$'s of sequence $S$ are generated by the following function. 
\begin{equation}
\label{eq:unrealizable_generate}
h(x_t)= 1\;\forall x_t
\end{equation}

That is, it assigns label $1$ to all the points. Since some $ x_t $'s of the sequence are negative and some of them are positive and neither function in our constructed hypothesis class $\mathcal{H}$ provides the label $ 1 $ to all the $x_t$'s of the sequence, this sequence will never be realizable by our $ \mathcal{H} $.

Table \ref{tab:unrealizable_generate} presents an example of a sequence generated by the function in equation \ref{eq:unrealizable_generate} with $T = 8$.
\end{enumerate}

\begin{table}[h]
	\centering
	\captionsetup{justification=centering,margin=1cm}
	\scalebox{1.3}{
	\begin{tabular}{| c | c | c | }
		\hline
		\rowcolor{lightgray}
		$t$ &  $x_t$ & $y_t$ \\ 
		\hline
		1.& $-3$ & $1$ \\  \hline
		2.& $-2$ & $1$ \\  \hline
		3.& $-1$ & $1$ \\  \hline
		4.& $0$ & $1$ \\  \hline
		5.& $1$ & $1$ \\  \hline
		6.& $2$ & $1$ \\  \hline
		7.& $3$ & $1$ \\  \hline
		8.& $4$ & $1$ \\  \hline
	\end{tabular}
	}
	\caption{Example sequence of length 8. This sequence is not realizable by our hypothesis class constructed according to hypotheses definition given by equation \ref{eq:H_generate}. The labels have been generated by hypothesis defined in equation \ref{eq:unrealizable_generate} which does not belong to our $\mathcal{H}$.}
	\label{tab:unrealizable_generate}
\end{table}

\begin{table*}[h!]
	\centering
	\captionsetup{justification=centering,margin=1cm}
	
	\scalebox{1.3}{
		\begin{tabular}{| c | c | c |c |c |c |c | c| }
			\hline
			\rowcolor{lightgray}
			$t$ &  $x_t$ &  $y_t$& $h_0(x_t)$ & $h_1(x_t)$  & $h_2(x_t)$ & $h_3(x_t)$  & $h_4(x_t)$ \\ 
			\hline
			1.& $-3$ & $0$ & $0$ & $0$& $0$& $0$& $0$\\  \hline
			2.& $-2$ & $0$ & $0$ & $0$& $0$& $0$& $0$\\  \hline
			3.& $-1$ & $0$ & $0$ & $0$& $0$& $0$& $0$\\  \hline
			4.& $0$  & $0$  & $0$ & $0$& $0$& $0$& $0$\\  \hline
			5.& $1$  & $1$  & $1$ &\textcolor{red}{0}&\textcolor{red}{0}&\textcolor{red}{0}&\textcolor{red}{0} \\  \hline
			6.& $2$  & $1$  & $1$ & $1$&\textcolor{red}{0}&\textcolor{red}{0}&\textcolor{red}{0} \\  \hline
			7.& $3$  & $1$  & $1$ & $1$& $1$&\textcolor{red}{0}&\textcolor{red}{0} \\  \hline
			8.& $4$  & $1$  & $1$ & $1$& $1$& $1$&\textcolor{red}{0}\\  \hline
			$\mathbf{M(h_i)}$ & $\mathbf{--}$  &$\mathbf{--}$ &\textcolor{red}{0} & \textcolor{red}{1}&\textcolor{red}{2}&\textcolor{red}{3}&\textcolor{red}{4} \\  \hline
		\end{tabular}
	}
	\caption{Showing predictions and total mistake count of each function over the entire sequence S.}
	\label{tab:unrealizable_performance}
\end{table*}

\end{enumerate}

\subsection{Simulation of WM{\_}Halving Algorithm }

As already mentioned, we are going to simulate only Algorithm \ref{alg:WM_Halving} named WM{\_}Halving.
From  theorem \ref{thm:WM_Halving}, Algorithm \ref{alg:WM_Halving} enjoys the following bounds:

\begin{enumerate}[(a)]
	\item In realizable case:
	$$M_{WM\_Halving}(\mathcal{H}, T) \leq \log_2(|\mathcal{H}|)-1$$
	
	\item In unrealizable case:

	$$\sum_{t=1}^T \mathds{E}[\mathbf{1}[\hat{y_t} \neq y_t]]- \min_{h \in \mathcal{H}} \sum_{t=1}^T \mathbf{1}[h(x_t) \neq y_t]  $$
	$$\leq  \sqrt{(0.5\,\ln(|\mathcal{H}|) \, (T-\log_2|\mathcal{H}|))} + $$ 
	$$ \text{log}_2(|\mathcal{H}|) $$

\end{enumerate}

In unrealizable case, we analyse the "expected" regret bound of learner which can be described in terms of expected mistake bound as follows :  
\vspace{0.5cm}

$\mathbf{Expected \,number\, of\, mistakes} $
	
$$ \mathbf{ = \frac{Total\, number\, of\, mistakes}{\splitfrac{Number\, of\, all\, possible\, permutations\, of} {the\, input\, sequence }}}$$ 
\vspace{1pt}

$ \mathbf{Expected\,Regret = Expected\,number\, of}$ 

$\mathbf{ mistakes -\,Number\, of\, mistakes\, made\, by\, best}$ $\mathbf{ functions} $


\vspace{1pt}
Here, we are interested in ``expected'' regret bound of the learner. Then, naturally, we need to count mistakes made by the learner over all possible permutations of an input sequence. It is extremely time consuming or in fact computationally infeasible to check the expected mistake count of any learning algorithm over all possible sequences of a long input sequence. Therefore, we have computed expected mistakes count on all possible permutations of very short sequence (e.g. $T = 8 \; \text{or} \; 9$) as well as we have also generated few random permutations of a long input sequence (when $T \sim 1000 \, \text{or} \, 10,000$).

\subsubsection*{\textbf{Simulation results}}

This section presents the simulation results for the realizable and unrealizable cases separately.

\begin{enumerate}[(a.)]
	\item \textit{Realizable case:}
	
	In realizable case, we assume that $\exists \, h^* \in \mathcal{H}$ such that
	$$h^*(x_t) = y_t\, ; \; \forall t \in \left[T\right]$$
	hence, $$ M_{h^*}(S) =  0 $$
	
	But depending on the order in which points of input sequence are presented to the algorithm, algorithm can make any number of mistakes ranging from 0 to $\log({\mathcal{H}})$.
	i.e.
	$$ M_A(H,S) \in \left[\textrm{log}_2({|\mathcal{H}|)}\right]$$
	
	These mistakes are essentially the mistakes made by the algorithm to find that best hypothesis $h^*$.
	
	Table \ref{tab:Sim_Realizable} presents some simulation results where sequence $S$ is realizable by the constructed hypothesis class $\mathcal{H}$ defined in equation \ref{eq:H_generate}.

	\begin{table*}[h!] 
		\centering
		\resizebox{0.88\textwidth}{!}
		{\begin{minipage}{\textwidth}
				\centering
				\captionsetup{margin=-0.5cm}
				    \scalebox{1.5}{
					\begin{tabular}{| c | c  c  c | c  c  c c | c | }\cline{1-9}
						\hline
						\rowcolor{lightgray}
						$T$ 
						&\vtop{\hbox{\strut $Permut-$}\hbox{\strut $ations$}}
						&  $|\mathcal{H}|$ 
						& $M(h^*)$      
						&\vtop{\hbox{\strut WM}\hbox{\strut expected}\hbox{\strut mistakes}}   
						&\vtop{\hbox{\strut WM}\hbox{\strut Max}\hbox{\strut mistakes}\hbox{\strut (1)}}  
						&\vtop{\hbox{\strut WM{\_}}\hbox{\strut Halving}\hbox{\strut expected}\hbox{\strut mistakes}}  
						&\vtop{\hbox{\strut WM{\_}}\hbox{\strut Halving}\hbox{\strut Max}\hbox{\strut mistakes}\hbox{\strut (2)}} 
						&\vtop{\hbox{\strut Diff}\hbox{\strut (1) - (2)}} \\ 
						\hline
						$1000$ & $100$ & $500$ & $0$  &$35.76$& $46$ & $3.25$ & $5$ & $41$ \\  \hline
						
						$1000$ & $100$ & $500$ & $0$ &$36.63$& $49$  & $3.31$ & $6$ & $43$ \\  \hline
						
						$1000$ & $100$ & $500$ & $0$ &$36.36$ & $46$ & $3.18$ & $6$ & $39$ \\  \hline
						
						$8$ & $40,320$ & $4$ & $0$ &$1.31$ & $3$ & $0.91$  & $2$ & $1$ \\  \hline
						
						$8$ & $40,320$ & $4$ & $0$  &$1.32$ & $3$ & $0.91$ & $2$ & $1$ \\  \hline
						
						$8$ & $40,320$ & $4$ & $0$ &$1.31$ & $3$ & $0.91$ & $2$ & $1$ \\  \hline

					\end{tabular}
					}
				\caption{Comparison of mistake counts of existing and proposed algorithms  in \textit{realizable case} on \textit{both} all permutations of the sequence S given in table \ref{tab:realizable_seq} and some randomly generated permutations of large sequence of 1000 points. In this table, $h^*$ denotes the best function. i.e. it does not make any mistake over the entire input sequence and it realizes the constructed hypothesis class $\mathcal{H}$.}
				\label{tab:Sim_Realizable}
			\end{minipage} 		
			}	
	\end{table*}
	
	\subsubsection*{Observations} 
	
	It can be seen that there is huge difference between the maximum mistakes made by $Weighted\, Majority$ and $WM{\_}Halving$ in realizable case and $WM{\_}Halving$ algorithm really outperforms in this case.
\vspace{0.3cm}	
	
	\item \textit{Unrealizable case:}
	
	In unrealizable case, we analyse the regret bound rather than the mistake bound. In other words, regret bound is essentially a mistake bound when mistakes are counted w.r.t the best hypothesis in the class.
	
	Table \ref{tab:Sim_Unrealizable} presents some simulation results where sequence $S$ is not realizable by the constructed hypothesis class $\mathcal{H}$ defined in equation \ref{eq:H_generate}. 
	
	\begin{table*}[h!] 
		\centering
		\resizebox{0.985\textwidth}{!} 
		{\begin{minipage}{\textwidth}
				\centering
				\captionsetup{margin=1cm}
				   \scalebox{1.4}{
					\begin{tabular}[t]{| c | c  c  c | c  c | c |}\cline{1-7}
						\hline
						\rowcolor{lightgray}
						$T$ 
						&\vtop{\hbox{\strut Permut-}\hbox{\strut ations}}
						& $|\mathcal{H}|$ 
						& $M(h^*)$     
						&\vtop{\hbox{\strut WM}\hbox{\strut expected}\hbox{\strut regret}\hbox{\strut (1)}}  
						&\vtop{\hbox{\strut WM{\_}}\hbox{\strut Halving}\hbox{\strut expected}\hbox{\strut regret}\hbox{\strut (2)}}    
						&\vtop{\hbox{\strut Diff}\hbox{\strut (1) - (2)}} \\ 
						\hline
						$1000$& $100$ & $500$ & $500$ & $36.15$ & $35.93$ & $0.22$ \\  \hline
						
						$1000$& $100$ & $500$ & $500$ & $36.86$ & $36.10$ & $0.76$ \\  \hline
						
						$1000$ & $100$ & $500$ & $500$ & $35.84$ & $35.84$ & $0$ \\  \hline
						
						$8$ & $40,320$ & $4$ & $4$ & $1.31$ & $1.28$ & $0.03$ \\  \hline
						
						$8$& $40,320$ & $4$ & $4$ & $1.32$ & $1.29$ & $0.03$ \\  \hline
						
						$8$& $40,320$ & $4$ & $4$ & $1.32$ & $1.28$ & $0.04$ \\  \hline
						
					\end{tabular}
				     }
					\caption{Comparison of mistake counts of existing and proposed algorithms  in \textit{unrealizable} case on \textit{both} all permutations of the sequence S given in table \ref{tab:unrealizable_generate} and some randomly generated permutations of large sequence of 1000 points generated in the same way. In this table, $h^*$ denotes the best function. i.e. which makes least number of mistakes over the entire input sequence.}
					\label{tab:Sim_Unrealizable}
			\end{minipage} 
		}	
	\end{table*}
	
	\subsubsection*{Observations}
	
	If we observe the values in "$Diff$" column in table \ref{tab:Sim_Unrealizable} and compare the "$Diff$" values from table \ref{tab:Sim_Realizable}, we see that we are still gaining very minimal or at least not loosing anything in the unrealizable case. But, in general, we are not guaranteed to gain anything in the unrealizable case using the proposed algorithms as stated in Theorem \ref{thm:WM_Halving}. This unexpected gain shows that, in general, we will lose very little in the unrealizable case.
	
	Note that regret values in this table are obtained by subtracting mistakes of best expert from the total mistakes made by the algorithms. i.e.\\ 
	\vspace{0.9 pt}
	\\ $$\textbf{Regret of learner} = \textbf{Total no. of mistakes by }  $$ $$\textbf{ learner - Total no. of mistakes by best function} $$
$\textbf{in}\; \mathcal{H}. $
	
	\vspace{0.3cm}
	For example, expected regret of WM in the first row is actually $500 + 36.15 = 536.15$
\end{enumerate}

\section{Conclusion and Future work}

In this work, we proposed three algorithms for the finite hypothesis class and both realizable and unrealizable cases. Proposed algorithms are designed by coupling the existing best algorithms available for realizable and unrealizable cases.

The motivation behind  proposed algorithms was to reduce the mistakes which $Weighted\, Majority$ makes in realizable case. Because, no matter the sequence is realizable or not, the expected regret bound was same. This was exploited by running $Weighted\, Majority$ in parallel. By doing so, if the input sequence is found to be realizable, algorithm will make very less mistake. If not, then we updated mistake count of each function in parallel which helped $Weighted\, Majority$ take over later on and predict optimally thereafter. 

The major contribution of this work is to propose algorithms which perform really outstanding in realizable case but slightly worse in unrealizable case. This nature of performance of the proposed algorithms is very useful in the scenarios where we are likely to get realizable input sequences. If input sequences are likely to be realizable than proposed algorithms will always be far better than the existing ones. 

The way in which $Weighted\, Majority$ runs in parallel indicates the scope of further improvement in theoretical bound of proposed algorithms in the unrealizable case. Further, the same approach has the potential to be applied in the other setting also of online learning; such as limited feedback model, stochastic noise model etc.

\section*{Acknowledgements}

First of all, we would like to thank the Almighty, who has always guided us to work on the right path of the life. We are also thankful to the entire faculty and staff members of Indian Statistical Institute (ISI) for their direct-indirect help, cooperation, love and affection. 

Last but not the least, we would like to thank our family
without whose blessings none of this would have been possible.

\appendices 
\section{Proof of the Weighted Majority Algorithm \cite{Shalev-Shwartz:2012}, \cite{Littlestone1994}}
\begin{theorem}
	Weighted Majority satisfies the following:
	$$\sum_{t=1}^T \mathds{E}[\mathds{1}[\hat{y_t} \neq y_t]]- \min_{i \in [d]} \sum_{t=1}^T \mathds{1}[h_i^t(x_t) \neq y_t] \leq \sqrt{0.5 \ln(d) T}$$ 
\end{theorem}

\begin{proof} \cite{Shalev-Shwartz:2012}
	The algorithm maintains the number of prediction mistakes each expert made so far in $M_i^{t-1}$ , and assign a probability weight to each expert accordingly. Then, the learner sets $\hat{p_t}$ to be the total mass of the experts which predict 1. The definition of $\hat{y_t}$ clearly implies that
	\begin{equation}
	\label{eq:WM_1}
	\mathds{E}[\mathds{1}[\hat{y_t} \neq y_t]] = \sum_{i=1}^d w_i^{t-1}\mathds{1}[h_i^t(x_t) \neq y_t]
	\end{equation}
	
	That is, the probability to make a mistake equals to the expected error of experts in $t^{th}$ round, where expectation is with respect to the probability vector $\mathbf{w}^t$.
	
	Now, we begin the proof.
	
	Define $Z_t = \sum_i e^{-\eta M_i^t}$. We have 
	$$\ln\frac{Z_t}{Z_{t-1}} = \ln \frac{\sum_i e^{-\eta M_i^{t-1}} e^{-\eta \mathds{1}[h_i^t(x_t) \neq y_t]}} {\sum_j e^{-\eta M_j^t}} $$
	
	$$= \sum_{i=1}^d w_i^{t-1} e^{-\eta \mathds{1}[h_i^t(x_t) \neq y_t]}$$
	
	Note that $\mathbf{w}^t$ is a probability vector and $\mathbf{1}[h_i^t(x_t) \neq y_t] \in [0,1]$. Therefore, we can
	apply Hoeffding’s inequality on the right-hand side of the above to get
	
	$$\ln\frac{Z_t}{Z_{t-1}} \leq -\eta \sum_{i=1}^d w_i^{t-1}\mathds{1}[h_i^t(x_t) \neq y_t]$$

	where the last equality follows from Eq. \ref{eq:WM_1}. Summing the above inequality over t we get
	\begin{equation}
	\label{eq:WM_2}
	\begin{split}
	\ln(Z_T)-\ln(Z_0) = \sum_{t=1}^T \ln \frac{Z_t}{Z_{t-1}} \leq -\eta \sum_{t=1}^T \mathds{E}[\mathds{1}[\hat{y_t} \neq y_t]]\\ + \frac{T\eta^2}{8}
	\end{split}
	\end{equation}

	Next, we note that $\ln(Z_0) = \ln(d)$ and that
	$$\ln Z_T = \ln (\sum_i e^{-\eta M_i^T}) \geq \ln (\max_i \, e^{-\eta M_i^T}) = -\eta \, \min_i \, M_i^T$$
	
	Substituting the values of $\ln(Z_T)$ and $\ln(Z_0)$ in Eq. \ref{eq:WM_2}
	
	$$ -\eta \, \min_i \, M_i^T - \ln(d) \leq -\eta \sum_{t=1}^T \mathds{E}[\mathds{1}[\hat{y_t} \neq y_t]]+ \frac{T\eta^2}{8}$$
	
	Dividing both sides by $\eta$ and rearranging the terms, we get 
	
	$$\sum_{t=1}^T \mathds{E}[\mathds{1}[\hat{y_t} \neq y_t]]- \min_{i \in [d]} \sum_{t=1}^T \mathds{1}[h_i^t(x_t) \neq y_t] \leq \frac{\ln(d)}{\eta} + \frac{\eta \, T}{8}$$
	
	Putting $\eta = \sqrt{8 \ln(d)/T}$, we get the desired result
	$$\sum_{t=1}^T \mathds{E}[\mathds{1}[\hat{y_t} \neq y_t]]- \min_{i \in [d]} \sum_{t=1}^T \mathds{1}[h_i^t(x_t) \neq y_t] \leq \sqrt{0.5 \ln(d) T}$$ 
	
\end{proof}

\ifCLASSOPTIONcaptionsoff
  \newpage
\fi



%

%
%

\bibliography{references}
\bibliographystyle{unsrt}

\end{document}